\documentclass[journal]{IEEEtran}

\usepackage[T1]{fontenc}
\usepackage{graphicx}
\usepackage{algorithm}
\usepackage{algorithmicx}
\usepackage{algpseudocode}
\usepackage{amsfonts}
\usepackage{amssymb}
\usepackage{color}
\usepackage{xcolor,colortbl}
\usepackage{mathtools} 
\usepackage{tikz}
\usepackage{capt-of}
\usepackage{multirow}
\usepackage{orcidlink}
\usepackage{amsmath}
\usepackage{subcaption}
\usepackage{eqparbox}
\usepackage{amsthm}
\usepackage{setspace}
\usepackage{marvosym} 

\usetikzlibrary{shapes,backgrounds,calc,bending,shapes.multipart}

\tikzset{split fill/.style args={#1 and #2}{path picture={
    \fill [#1] (path picture bounding box.south west)
      rectangle (path picture bounding box.north);
    \fill [#2] (path picture bounding box.south)
      rectangle (path picture bounding box.north east);
}}} 
\DeclareMathOperator*{\argmax}{arg\,max}

\newtheorem{corollary}{Corollary}
\newtheorem{definition}{Definition}
\newtheorem{example}{Example}
\newtheorem{lemma}{Lemma}
\newtheorem{proposition}{Proposition}

\newtheorem{theorem}{Theorem}

\newcommand{\p}{\mathbb{P}}

\newcommand{\ext}[1]{\mathcal{E}_{#1}}

\newcommand{\Tobs}{D}

\newcommand{\hyp}{\mathcal{M}}

\newcommand{\dom}{[\feat]^{\leq 2}}

\newcommand{\entropy}{H}

\newcommand{\feat}{\mathcal{F}}

\newcommand{\Hnorm}{\eta}

\newcommand{\dfpath}{\textit{df}-path}

\newcommand{\E}{E}

\newcommand{\X}{\mathcal{X}}

\newcommand{\RI}{\overset{\texttt{Invoice}}{\sim}}
\newcommand{\RO}{\overset{\texttt{Order}}{\sim}}
\newcommand{\RP}{\overset{\texttt{Payment}}{\sim}}
\newcommand{\ROP}{\overset{\X}{\sim}}

\definecolor{Gray}{gray}{0.85}
\definecolor{LightCyan}{rgb}{0.88,1,1}

\ifCLASSINFOpdf
\else
\fi
\begin{document}

%


%
%
%

\title{Posets and Bounded Probabilities for\\Discovering Order-inducing Features\\in Event Knowledge Graphs}


\author{{Christoffer~Olling~Back\orcidlink{0000-0001-7998-7167}(\Letter) and Jakob~Grue~Simonsen\orcidlink{0000-0002-3488-9392}}
\thanks{C. O. Back was with the Department of Computer Science, University of Copenhagen, Denmark and ServiceNow Denmark Aps, Aarhus, Denmark. \mbox{e-mail: back@di.ku.dk.}}
\thanks{J. G. Simonsen was with the Department of Computer Science, University of Copenhagen, Denmark and the IT University of Copenhagen, Denmark.}
\thanks{Supported by Innovation Fund Denmark as part of DIREC initiative.}}
\markboth{}{}
%



\maketitle

\begin{abstract}
%
%
Event knowledge graphs (EKG) extend the classical notion of a trace to capture multiple, interacting views of a process execution.
In this paper, we tackle the open problem of automating EKG discovery from uncurated data through a principled probabilistic framing based on the outcome space resulting from featured-derived partial orders on events.
From this we derive an EKG discovery algorithm based on statistical inference rather than an ad hoc or heuristic-based strategy, or relying on manual analysis from domain experts.

This approach comes at the computational cost of exploring a large, non-convex hypothesis space.
In particular, solving the maximum likelihood term in our objective function involves counting the number of linear extensions of posets, which in general is \#P-complete.
Fortunately, bound estimates suffice for model comparison, and admit incorporation into a bespoke branch-and-bound algorithm.
We establish an upper bound on our objective function which we show to be antitonic w.r.t.~search depth for branching rules that are monotonic w.r.t.~model inclusion.
This allows pruning of large portions of the search space, which we show experimentally leads to rapid convergence toward optimal solutions that are consistent with manually built EKGs.


\end{abstract}


%
\IEEEpeerreviewmaketitle

\section{Introduction}\label{sec:introduction}
\IEEEPARstart{E}{vent} knowledge graphs (EKGs) describe the temporal ordering of events associated with entities and objects via so-called local directly-follows paths (\dfpath s), resulting in an intertwined web of event sequences synchronized at shared events.
Automating the process of identifying, ``relevant structural relations between entities'' is a stated open problem, currently requiring domain knowledge~\cite{esser_multi-dimensional_2021,van_der_aalst_process_2022}.

  \tikzset{%
    diagonal fill/.style 2 args={%
        double color fill={#1}{#2},
        shading angle=45,
        opacity=1.0},
    other filling/.style={%
        shade,
        shading=myshade,
        shading angle=0,
        opacity=0.5}
   }

\begin{table*}[t]
\begin{minipage}[b]{.55\textwidth }%
\caption{Event table in chronological order. Colors indicate grouping w.r.t.~\texttt{Order} and \texttt{Invoice}. Adapted from \cite{van_der_aalst_process_2022}.}\label{tab:event_table}
\footnotesize\centering
\begin{tabular}{c| l | c | c | c | c c | c}
\hline
\multicolumn{1}{c|}{\rotatebox{90}{\texttt{Event}}} & \multicolumn{1}{c|}{\rotatebox{90}{\texttt{Activity}}} & \multicolumn{1}{c|}{\rotatebox{90}{\texttt{Actor}}} & \multicolumn{1}{c|}{\rotatebox{90}{\texttt{Order}}} & \multicolumn{1}{c|}{\rotatebox{90}{\texttt{Supplier}}\rotatebox{90}{\texttt{Order}}} & \multicolumn{2}{c|}{\rotatebox{90}{\texttt{Invoice}}} & \multicolumn{1}{c}{\rotatebox{90}{\texttt{Payment}~}}\\
\hline
    $e_1$~ & Create Order & $R_1$ & \cellcolor{cyan!30}$O_1$ &   && &\\
    $e_2$~ & Create Order & $R_1$ & \cellcolor{cyan!100}$O_2$ &   && &\\
    $e_3$~ & Place S.O. & $R_1$ &    & $A$ && &\\
    $e_4$~ & Place S.O. & $R_3$ &    & $B$ && &\\
    $e_5$~ & Create Invoice & $R_3$ & \cellcolor{cyan!100}$O_2$ &   & \multicolumn{2}{c|}{ \cellcolor{yellow!30}$I_2$} & \\
    $e_6$~ & Receive S.O. & $R_2$ &    & $A$ && &\\
    $e_7$~ & Update S.O. & $R_1$ & \cellcolor{cyan!100}$O_2$ & $B$ &&  &\\
    $e_9$~ & Update Invoice & $R_2$ &    & $A$ & \multicolumn{2}{c|}{ \cellcolor{yellow!30}$I_2$} &\\
    $e_{18}$ & Create Invoice & $R_3$ & \cellcolor{cyan!30}$O_1$ &   & \multicolumn{2}{c|}{ \cellcolor{yellow!100}$I_1$} &\\
    $e_{19}$ & Receive S.O. & $R_2$ &    & $B$ &&  &\\
    $e_{28}$ & Ship Order & $R_4$ & \cellcolor{cyan!30}$O_1$ &   &&  &\\
    $e_{29}$ & Receive Payment & $R_5$ & & & & & $P_1$ \\
    $e_{30}$ & Clear Invoice & $R_5$ & & & \cellcolor{yellow!100}~$I_1$~ & \cellcolor{yellow!30}$I_2$ & $P_1$\\
    $e_{34}$ & Ship Order & $R_4$ & \cellcolor{cyan!100}$O_2$ &   && &\\
    \hline
\end{tabular}
\end{minipage}%
\begin{minipage}{.35\textwidth}
\centering
\begin{tikzpicture}[->]
    \tikzstyle{event}=[circle,draw,minimum size=15pt,inner sep=0pt]

    \node[event,fill=cyan!70] (e2) at (0,3) {$e_2$};
    \node[event,split fill=cyan!70 and yellow!30] (e5) at (0,2) {$e_5$ \nodepart{lower} $e_5$};
    \node[event,fill=cyan!70] (e7) at (0,1) {$e_7$};
    \node[event,fill=cyan!70] (e34) at (0,0) {$e_{34}$};

    \node[event] (e3) at (0,4) {$e_3$};
    \node[event] (e4) at (1,4) {$e_4$};
    \node[event] (e6) at (2,4) {$e_6$};
    \node[event,fill=yellow!30] (e9) at (1,1) {$e_9$};
    \node[event,split fill=yellow!30 and yellow!100] (e30) at (1,0) {$e_{30}$};

    \node[event] (e19) at (1,3) {$e_{19}$};
    \node[event] (e29) at (2,3) {$e_{29}$};
    \node[event,fill=cyan!30] (e1) at (2,2) {$e_1$};
    \node[event,split fill=yellow!100 and cyan!30] (e18) at (2,1) {$e_{18}$};
    \node[event,fill=cyan!30] (e28) at (2,0) {$e_{28}$};

    \draw (e2) -- (e5);
    \draw (e5) -- (e7);
    \draw (e7) -- (e34);

    \draw (e1) -- (e18);
    \draw (e18) -- (e28);

    \draw (e5) -- (e9);
    \draw (e9) -- (e30);
    \draw (e18) -- (e30);
\end{tikzpicture}
\smallskip
\captionof{figure}{Transitive reduction of the poset induced by feature relations in Tab.~\ref{tab:event_table}.}
\label{fig:poset-running}
\end{minipage}
\end{table*}

In this paper, we present a theoretically principled method for automatically identifying which event-level features should be included as entities in the EKG, thereby defining the structure of its \dfpath s.
Our approach rests on a natural Bayesian formulation which reduces the task to its core elements.
In particular, features that are candidates to become included as entities are analyzed w.r.t.~the \dfpath s that they induce, which taken together define a partial order on the entire set of events (i.e.,~a poset).
The resulting poset then becomes the cornerstone of a probabilistic approach to model comparison.
This paper extends~\cite{back_discovering_2025} with additional propositions, lemmas, proofs, examples, visualizations, and EKG comparisons for evaluation; as well as detailed pseudocode descriptions of the core procedures in our approach.

We take a maximum likelihood approach, seeking the model that maximizes the probability of the observed evidence, namely the sequence of events in a given sample.
This is by definition one linear extension permitted by the event poset, so the task then becomes to identify those features that most efficiently narrow the outcome space while permitting the observed sequences.

The maximum likelihood formulation requires some limiting factor to avoid arbitrarily complex models~\cite{mackay_information_2003}, and to this end we introduce an entropy-based measure of \emph{feature informativeness}.
This measure can be interpreted as a prior probability over models that clearly favors informative models and disfavors models that are either trivially simple or add unnecessary complexity.
These two terms -- likelihood and model prior -- taken together define the posterior probability of a model given observed data, weighted by model informativeness.

A challenge arises in calculating the likelihood term, as it relies on the number of linear extensions of the event poset.
In general, counting linear extensions is \#P-complete~\cite{brightwell_counting_1991}, but we show that poset structure can be exploited to establish bounds on the posterior, by extension posterior odds between models.
This is sufficient, as we are only interested in comparing models, not exact computation of the posterior.
Finally, we incorporate this model comparison strategy into a bespoke branch and bound algorithm for exploring feature combinations that can prune large portions of the search space, making the overall task tractable.

In the remainder, we review related work in Sec.~\ref{sec:related-work}, introduce notation and definitions in Sec.~\ref{sec:preliminaries}, and formalize EKGs as posets in Sec.~\ref{sec:event-knowledge-graph-as-partial-order}.
We present our Bayesian problem formulation in Sec.~\ref{sec:a-probabilistic-formulation} and outline the branch and bound strategy in Sec.~\ref{sec:strategies-bounds-and-model-selection}.
We present experimental results in Sec.~\ref{sec:evaluation}, and conclude in Sec.~\ref{sec:conclusion} with suggestions for future work.

\section{Related Work}\label{sec:related-work}
Event knowledge graphs were first introduced by Fahland et al.~\cite{esser_multi-dimensional_2021,van_der_aalst_process_2022,polyvyanyy_classifying_2021}, and reflect a growing recognition in the process mining community of the shortcomings of the classical trace concept based on a single case identifier~\cite{adams_defining_2022,olling_back_stochastic_2021,olveczky_object-centric_2019}.
Even in scenarios where a classical event log is sufficient, the process of extracting event traces from the raw data source still tends to be an expensive, ad-hoc process relying on scarce domain expert input~\cite{berti_generic_2023}.
While some work has been done on automating this task for classical event logs~\cite{diba_extraction_2020}, for novel representations like EKGs, automation remains largely unaddressed~\cite{van_der_aalst_process_2022}.

The task of calculating the number of linear extensions of a poset has been explored by mathematicians for decades as a combinatorics problem in its own right~\cite{butler_number_1972}; as well as by computer scientists, typically in the context of sorting and other comparison-based algorithms~\cite{aigner_producing_1981,brightwell_counting_1991,schellekens_entropy_2020}.
It also arises in a number of other applications such as convex rank tests~\cite{morton_convex_2009}, structure learning for graphical models~\cite{niinim_structure_2016}, measuring flexibility of partial-order plans~\cite{muise_optimal_2016}, sequence analysis~\cite{mannila_global_2000}, and is equivalent to counting topological sorts on an acyclic graph~\cite{robinson_counting_1977}.

Branch and bound algorithms have long been used for feature selection in general~\cite{chen_improved_2003,nakariyakul_adaptive_2007,ring_optimal_2015,thakoor_branch-and-bound_2011,yu_more_1993}, and model selection~\cite{thakoor_branch-and-bound_2008,thakoor_branch-and-bound_2011} and maximum likelihood estimation in particular~\cite{decani_branch_1972,zelniker_optimal_2005}.

\section{Preliminaries}\label{sec:preliminaries}
The reader is expected to have a basic familiarity with combinatorics, order theory, probability theory, and set theory.

The set of possible permutations of the elements of a set $S$ is denoted $S!$ and the symmetric difference between sets $A$ and $B$ as $A~\triangle~B$.
We denote by $[S]^{\leq n}$ the set of subsets of $S$ containing $n$ or fewer elements, that is: $\{ S' \subseteq S \mid n \geq |S'|\}$.

An event is a unique execution of some action (i.e.,~activity) at a certain point in time, and associated with (potentially empty) attributes, (i.e.,~features).
In the event table in Tab.~\ref{tab:event_table} each event is represented by a row.

We denote by $\E$ a nonempty set of events, by $D$ a total order on $E$ called the observed data, by $\mathcal{D} \coloneqq \{D_1,\ldots,D_N\}$ a set of $N$ independent observations.
We denote by $X_i$ a feature associated with events, i.e.,~a column in an event table, and by $X_i(e)$ the set of values of $X_i$ associated with event $e$.
The set of all features is denoted by $\feat \coloneq \{X_1,\ldots,X_M\}$.
Singleton sets of features are called \emph{atomic}, whereas sets containing exactly two features are called \emph{derived}, following~\cite{van_der_aalst_process_2022}.
These \emph{sets} of features are denoted $\X$, with $\X \in [\feat]^{\leq 2}$, and constitute the building blocks of models and, by extension, partial orders.
We denote by $\overset{\X}{\sim}$ a feature relation between events as defined in Def.~\ref{def:feature-relation} and~\ref{def:derived-feature-relation}.

A model $\hyp \subseteq [\feat]^{\leq 2}$ is a collection of sets of features.
For example, the model $\{\{X_1\},\{X_2,X_3\}\}$ contains one atomic and one derived feature.
An ordering relation $\prec$ when paired with a set, $P \coloneq (E, \prec)$, is called a partially ordered set or poset.
The set of linear extensions of poset $P$ is written $\ext{P}$, and represents all total orders on $E$ that respect the partial order $\prec$.
We are particularly interested in the number of linear extensions, written $|\ext{P}|$.
The function $g_\hyp(-)$ takes a partial order and reduces it according to a model $\hyp$, see Def.~\ref{def:dfpath-generator}.

Finally, we write the probability of statement $\phi$ being true as $\p(\phi)$; for example,~$\p(X_i = x)$ or $\p(\Tobs \in \ext{P})$.
 The outcome space of the stochastic variable $X$ is written $\Omega_X$ and we write the Shannon entropy of $X$ as $\entropy(X)$ and its normalized variant as $\Hnorm(X)$.

\section{Event Knowledge Graph as Poset}\label{sec:event-knowledge-graph-as-partial-order}
EKGs are not per se process models with execution semantics, but rather a representation of one observed outcome: one execution trace.
Each event is a specific, unique occurrence that will never be duplicated, and the local directly-follows relations between them are not, strictly speaking, order relations in the sense of rules or constraints of a \emph{model} that defines a language of traces.

Nevertheless, we will argue that EKGs are functionally equivalent to posets.
After establishing definitions, we show how an observed order is reduced w.r.t.~ a model.
In this case, the original poset will be a total order (an event table) and the resulting poset interpreted as the \dfpath s of an EKG\@.

    \begin{figure*}
        \begin{center}
            \begin{tikzpicture}[->]
    \tikzstyle{event}=[circle,minimum size=12pt,inner sep=0pt]
    \foreach \col/\i [count=\x from 0] in {white/1, gray/2, white/3, white/4, gray/5, gray/6, white/7, white/8}
        \node[draw,event,fill=\col!50] (T1-\x) at (0.65*\x,1.2) {$\i$};
    \foreach \col/\i [count=\x from 0] in {gray/2, white/1, white/3, gray/5, white/4, white/7, gray/6, white/8}
        \node[draw,event,fill=\col!50] (T2-\x) at (0.65*\x,0.6) {$\i$};
    \foreach \col/\i [count=\x from 0] in {white/1, white/3, white/4, white/7, white/8, gray/2, gray/5, gray/6}
        \node[draw,event,fill=\col!50] (T3-\x) at (0.65*\x,0) {$\i$};

    \foreach \x [count=\y] in {0,...,6} \draw (T1-\x) -- (T1-\y);
    \foreach \x [count=\y] in {0,...,6} \draw (T2-\x) -- (T2-\y);
    \foreach \x [count=\y] in {0,...,6} \draw (T3-\x) -- (T3-\y);

    \foreach \col/\i [count=\x from 2] in {gray/2, gray/5, gray/6}
        \node[draw,event,fill=\col!50] (P-\i) at (8+0.7*\x,0.95) {$\i$};
    \foreach \col/\i [count=\x] in {white/1, white/3, white/4, white/7, white/8}
        \node[draw,event,fill=\col!50] (P-\i) at (8+0.7*\x,0.15) {$\i$};

        {
    \draw (P-2) -- (P-5);\draw (P-5) -- (P-6);
    \draw (P-1) -- (P-3);\draw (P-3) -- (P-4);
    \draw (P-4) -- (P-7);\draw (P-7) -- (P-8);
        }

    \node[fill=white!0,minimum size=0pt] (T1) at (5.1,1.2) {};
    \node[fill=white!0,minimum size=0pt] (T2) at (5.1,0.6) {};
    \node[fill=white!0,minimum size=0pt] (T3) at (5.1,0) {};
    \node[fill=white!0,minimum size=5pt] (g1) at (6.2,1.0) {$g_\hyp$};
    \node[fill=white!0,minimum size=5pt] (g2) at (6.2,0.6) {$g_\hyp$};
    \node[fill=white!0,minimum size=5pt] (g3) at (6.2,0.2) {$g_\hyp$};
    \node[fill=white!0,minimum size=1pt] (image) at (8.4,0.6) {};

        {
    \draw (T1) -- (g1);
    \draw (T2) -- (g2);
    \draw (T3) -- (g3);
    \draw (g1) -- (image);
    \draw (g2) -- (image);
    \draw (g3) -- (image);
        }
\end{tikzpicture}
        \end{center}
        \caption{The \dfpath~generator $g_\hyp$ is in general not injective: it will typically map several different partial orders over the same set and model $\hyp$ to one new partial order.}\label{fig:df-path-generator}
    \end{figure*}

\begin{definition}[Atomic feature relation]
    Two events $a,b \in \E$ are related w.r.t.~feature $\X \coloneqq \{X\}$ if they share one or more of the values for that feature. Formally,
\[
\overset{X}{\sim}~\coloneqq \left\{~(a,b) \mid X(a) \cap X(b) \neq \emptyset~\right\}
\]
    \label{def:feature-relation}
\end{definition}
\begin{example}
    The relation for atomic feature $\{\texttt{Invoice}\}$ in Tab.~\ref{tab:event_table} is defined as follows.
    Events $e_5, e_{9}, e_{30}$ share value $I_2$:
    \begin{align*}
    \texttt{Invoice}(e_5)~\cap~\texttt{Invoice}(e_{9}) =~~~& \\ \texttt{Invoice}(e_9)~\cap~\texttt{Invoice}(e_{30})& = \\
    \texttt{Invoice}(e_{30})~\cap~\texttt{Invoice}&(e_{5}) = \{I_2\}
    \end{align*}
    while events $e_{18}$ and $e_{30}$ share value $I_1$: 
    \[\texttt{Invoice}(e_{18})~\cap~\texttt{Invoice}(e_{30}) = \{I_1\}.\]
    This results in the relations: 
    \[e_5 \RI e_9, e_9 \RI e_{30},\]
    \[e_5 \RI e_{30}, \text{~and~} e_{18} \RI e_{30}.\]
\end{example}

\begin{definition}[Derived feature relation]
    Two events $a,d \in \E$ are related w.r.t.~derived feature $\X \coloneqq \{X_i,X_k\}$ if they are related by $X_i$ or $X_k$, or if there exists a third atomic feature $X_j\in\feat \setminus \{X_i,X_k\}$ and distinct events $b,c \in \E$ s.t. $a$ and $d$ are related transitively via $b$ and $c$ by relations $\overset{X_i}{\sim}$,$\overset{X_j}{\sim}$, and $\overset{X_k}{\sim}$.
    Formally,
\begin{align*}
\overset{X_i,X_k}{\sim}&~\coloneqq~\overset{X_i}{\sim} \cup \overset{X_k}{\sim} \cup~ \\
&\left\{(a,d) \mid \exists X_j,b,c.~X_j\notin \{X_i,X_k\} \land a \overset{X_i}{\sim} b\overset{X_j}{\sim}c \overset{X_k}{\sim}d\right\}
\end{align*}
    \label{def:derived-feature-relation}
\end{definition}
\begin{example}
    The relation for derived feature $\X \coloneq \{\texttt{Order},\texttt{Payment}\}$ in Tab.~\ref{tab:event_table} is defined as follows.
    Letting $X_j = \texttt{Invoice}$ and considering those events that are atomically related via value $I_1$, namely $e_{18}$ and $e_{30}$, we have
    \[e_1\RO e_{18}\RI e_{30} \RP e_{29}\]
    and
    \[e_{28}\RO e_{18}\RI e_{30} \RP e_{29}\]
    in addition to variants of these resulting from the reflexive relations: 
    \[e_{18}\RO e_{18} \textrm{~and~} e_{30} \RP e_{30}.\]
    Once events related by $I_2$ are included, in addition to the  atomic relations in $\RO$ and $\RP$, we have 
    \[e_1 \ROP e_{29},~~e_1 \ROP e_{30},~~e_2 \ROP e_{29},~~e_2 \ROP e_{30},~~e_{34} \ROP e_{29},\]
    \[e_5 \ROP e_{29},~~e_5 \ROP e_{30},~~e_7 \ROP e_{29},~~e_7 \ROP e_{30},~~e_{34} \ROP e_{30},\]
    \[e_{18} \ROP e_{29}, ~~e_{18} \ROP e_{30}, ~~e_{28} \ROP e_{29}, \text{~~and~~} e_{28} \ROP e_{30}.\]
\end{example}

In the remainder, we will write $\overset{\X}{\sim}$ to denote feature relations in general, following Def.~\ref{def:feature-relation} if $|\X|=1$ and Def.~\ref{def:derived-feature-relation} if $|\X|=2$.

\begin{definition}[\dfpath~generator]
    Given a partial order $D\coloneqq(E,\prec)$ and model $\hyp$, then the \dfpath~\emph{generator} is
    
    \begin{equation*}
        g_\hyp(D) \coloneqq \left\{~(a, b)~\mid~ a \prec b ~\land\ \exists \X\in\hyp.~a \overset{\X}{\sim}b~\right\}
    \end{equation*}
    \label{def:dfpath-generator}
\end{definition}

This means that order relations in $D$ are included in $g_\hyp(D)$ if the elements are also feature-related w.r.t. some feature in model $\hyp$.

\begin{example}
    Letting $\hyp \coloneqq \{\{\texttt{Invoice}\}\}$, denoting by $D$ the event table in Tab.~\ref{tab:event_table}
    \[
        g_\hyp(D) = \{(e_5,e_9),(e_9,e_{30}),(e_5,e_{30}),(e_{18},e_{30})\}
    \]
    The application of the \dfpath~generator to the model $\{\{\texttt{Order}\},\{\texttt{Invoice}\}\}$ is depicted in Fig.~\ref{fig:poset-running}, with transitive relations omitted for readability.
\end{example}

\subsection{Linear extensions as outcome space}

We now show that regardless of whether EKGs are intended to be seen as process models, many hypothetical orderings of the same set of events map to the same set of order relations (i.e.,~\dfpath s) via $g_\hyp$, meaning the \emph{preimage} of $g_\hyp$ is formally an outcome space.
This is illustrated in Fig.~\ref{fig:df-path-generator}.

\begin{proposition}
    Given posets $D \coloneqq (E, \prec)$ and $D'\coloneqq (E, \prec')$ over the same set $E$, a model $\hyp(E)$, and \dfpath~generator $g_\hyp$, then $D$ extends $g_\hyp(D')$ and $D'$ extends $g_\hyp(D)$ if and only if $g_\hyp(D) = g_\hyp(D')$.
\end{proposition}

\begin{proof}
    First, assume $g_\hyp(D) = g_\hyp(D')$ and denote this set $g_\hyp$ for convenience.
    That $D$ and $D'$ extend $g_\hyp$ follows immediately since for all $(a,b) \in g_\hyp$, we know that $a \prec b$ in $D$ and $D'$, by Def.~\ref{def:dfpath-generator}.
    Furthermore, any $(c,d)$ in $D$ or $D'$ and not in $g_\hyp$, we know that $(d,c) \notin g_\hyp$ by the antisymmetry property of partial orders.

    Second, assume that $D$ and $D'$ extend $g_\hyp(D')$ respectively $g_\hyp(D)$, but that $g_\hyp(D') \neq g_\hyp(D)$.
    This means there exists some $(a,b) \in g_\hyp(D)~\triangle~g_\hyp(D')$.
    It follows from Def.~\ref{def:dfpath-generator} that $a \nprec b$ in $D$ or $D'$ either because $b \prec a$, in which case $(b,a) \in g_\hyp(D)~ \triangle~g_\hyp(D')$, or because $a$ and $b$ are incomparable, meaning either $D$ does not extend $g_\hyp(D')$ or $D'$ does not extend $g_\hyp(D)$, and we have a contradiction.
\end{proof}

\section{A Probabilistic (Bayesian) Formulation}\label{sec:a-probabilistic-formulation}
    \renewcommand{\arraystretch}{1.15}
\begin{table*}[ht]
    \caption{Examples of various combinations of features, their $\eta$-based model prior $\p(\hyp)$ (see Def.~\ref{def:model-prior}), and the resulting poset. The model $\{\{X_1\}\}$ is degenerate because it maps exactly to original total order. The model $\{\{X_2\}\}$ has a high prior due to its high entropy, and induces a moderately restrictive order. Adding $\{X_3\}$ to this model adds no new structure to the order, but lowers the prior. In contrast, adding $\{X_4\}$ connects previously disjoint components, improving the maximum likelihood and potentially outweighing the decrease in $\p(\hyp)$.}
    \centering
    \begin{tabular}{c|c|c|c|c|r|c|c|c|c|c|c}
        \hline
        \multicolumn{5}{c}{Event Table} & \multicolumn{1}{c}{$\Hnorm$-based Model Prior $\p(\hyp)$ (numerator)} & \multicolumn{6}{c}{Induced Posets} \\
        \hline
        & $X_1$ & $X_2$ & $X_3$ & $X_4$ & & ${X_1}$ & ${X_2}$ & ${X_3}$ & ${X_4}$  & ${X_{2,3}}$ & ${X_{2,4}}$ \\
        \hline
            \vspace{-9pt} & & & & & & & & & & & \\
            $e_1$ & A & B & F &   & \multirow{2}*{$\Hnorm(X_1) = \frac{-1 \log 1}{\log 8+1} = 0$} & \multirow{8}*{\begin{tikzpicture}[->]
    
    \tikzstyle{event}=[circle,draw,minimum size=6pt,inner sep=0.5pt]
    \node[event] (e1) at (0,3.5) {};
    \node[event] (e2) at (0,3.0) {};
    \node[event] (e3) at (0,2.5) {};
    \node[event] (e4) at (0,2.0) {};
    \node[event] (e5) at (0,1.5) {};
    \node[event] (e6) at (0,1.0) {};
    \node[event] (e7) at (0,0.5) {};
    \node[event] (e8) at (0,0.0) {};

    \draw (e1) -- (e2);
    \draw (e2) -- (e3);
    \draw (e3) -- (e4);
    \draw (e4) -- (e5);
    \draw (e5) -- (e6);
    \draw (e6) -- (e7);
    \draw (e7) -- (e8);
    
\end{tikzpicture}} & \multirow{8}*{\begin{tikzpicture}[->]
    
    \tikzstyle{event}=[circle,draw,minimum size=6pt,inner sep=0.5pt]
    \node[event] (e1) at (0,3.5) {};
    \node[event] (e2) at (0,3.0) {};
    \node[event] (e3) at (0,2.5) {};
    \node[event] (e4) at (0,2.0) {};
    \node[event] (e5) at (0,1.5) {};
    \node[event] (e6) at (0,1.0) {};
    \node[event] (e7) at (0,0.5) {};
    \node[event] (e8) at (0,0.0) {};

    \draw (e1) -- (e2);
    \draw (e3) -- (e4);
    \draw (e5) -- (e6);
    \draw (e7) -- (e8);
    
\end{tikzpicture}} & \multirow{8}*{\begin{tikzpicture}[->]
    
    \tikzstyle{event}=[circle,draw,minimum size=6pt,inner sep=0.5pt]
    \node[event] (e1) at (0,3.5) {};
    \node[event] (e2) at (0,3.0) {};
    \node[event] (e3) at (0,2.5) {};
    \node[event] (e4) at (0,2.0) {};
    \node[event] (e5) at (0,1.5) {};
    \node[event] (e6) at (0,1.0) {};
    \node[event] (e7) at (0,0.5) {};
    \node[event] (e8) at (0,0.0) {};

    \draw (e1) -- (e2);
    \draw (e5) -- (e6);
    
\end{tikzpicture}} & \multirow{8}*{\begin{tikzpicture}[->]
    
    \tikzstyle{event}=[circle,draw,minimum size=6pt,inner sep=0.5pt]
    \node[event] (e1) at (0,3.5) {};
    \node[event] (e2) at (0,3.0) {};
    \node[event] (e3) at (0,2.5) {};
    \node[event] (e4) at (0,2.0) {};
    \node[event] (e5) at (0,1.5) {};
    \node[event] (e6) at (0,1.0) {};
    \node[event] (e7) at (0,0.5) {};
    \node[event] (e8) at (0,0.0) {};

    \draw (e2) -- (e3);
    \draw (e6) -- (e7);
    
\end{tikzpicture}} & \multirow{8}*{\begin{tikzpicture}[->]
    
    \tikzstyle{event}=[circle,draw,minimum size=6pt,inner sep=0.5pt]
    \node[event] (e1) at (0,3.5) {};
    \node[event] (e2) at (0,3.0) {};
    \node[event] (e3) at (0,2.5) {};
    \node[event] (e4) at (0,2.0) {};
    \node[event] (e5) at (0,1.5) {};
    \node[event] (e6) at (0,1.0) {};
    \node[event] (e7) at (0,0.5) {};
    \node[event] (e8) at (0,0.0) {};

    \draw (e1) -- (e2);
    \draw (e3) -- (e4);
    \draw (e5) -- (e6);
    \draw (e7) -- (e8);
    
\end{tikzpicture}} & \multirow{8}*{\begin{tikzpicture}[->]
    
    \tikzstyle{event}=[circle,draw,minimum size=6pt,inner sep=0.5pt]
    \node[event] (e1) at (0,3.5) {};
    \node[event] (e2) at (0,3.0) {};
    \node[event] (e3) at (0,2.5) {};
    \node[event] (e4) at (0,2.0) {};
    \node[event] (e5) at (0,1.5) {};
    \node[event] (e6) at (0,1.0) {};
    \node[event] (e7) at (0,0.5) {};
    \node[event] (e8) at (0,0.0) {};

    \draw (e1) -- (e2);
    \draw (e2) to [bend left=70] (e3);
    \draw (e3) -- (e4);
    \draw (e5) -- (e6);
    \draw (e6) to [bend left=70] (e7);
    \draw (e7) -- (e8);
    
\end{tikzpicture}}  \\
            $e_2$ & A & B & F & H &  & & & & &\\
            $e_3$ & A & C &   & H & \multirow{2}*{$\Hnorm(X_2) = \frac{-\frac{4}{4} \log \frac{1}{4}}{\log 8 + 1} = \frac{1}{2}$} & & & & &\\
            $e_4$ & A & C &   &   &  & & & & &\\
            $e_5$ & A & D & G &   & \multirow{2}*{$\Hnorm(X_3)=\Hnorm(X_4)=\frac{-\frac{1}{2} \log \frac{1}{2} -\frac{1}{2} \log \frac{1}{4}}{\log 8 + 1} = \frac{3}{8}$} & & & & &\\
            $e_6$ & A & D & G & I & & & & & &\\
            $e_7$ & A & E &   & I & \multirow{2}*{$\Hnorm(X_2) \Hnorm(X_3) = \Hnorm(X_2) \Hnorm(X_4) = \frac{3}{16}$} & & & & &\\
            $e_8$ & A & E &   &   & & & & & & \\
        \hline
    \end{tabular}
    \label{tab:model-prior-examples}
\end{table*}
\renewcommand{\arraystretch}{1}

Our aim is to identify the set of features $\hyp$ that provides the most information about the observed sequence of events w.r.t.~the poset induced by $\hyp$ via \dfpath~generator $g_\hyp$.
We now formulate this problem at a high level in probabilistic terms.
We seek to find the model $\hyp$ which is most strongly indicated by our data $D$, i.e.:~ $\argmax_{\hyp}~\p(\hyp\mid D)$.
The term $\p(\hyp\mid D)$ is the posterior probability of model $\hyp$ after observing data $D$; by Bayes' rule, this can be rewritten as:

\begin{equation*}
    \argmax_{\hyp}~\p(\hyp\mid D) = \argmax_{\hyp}~\frac{\p(\hyp)~\p(D\mid \hyp)}{\p(D)}
    \label{eq:equation}
\end{equation*}

where $\p(D\mid\hyp)$ is the \emph{likelihood} of $\hyp$ while $\p(\hyp)$ is its \emph{prior} probability, and $\p(D)$ is the \emph{evidence}.
Once $D$ has been observed it is fixed, and it will turn out to cancel out in our final, posterior-odds objective function for comparing models (see Sec.~\ref{subsec:objective-function}).
Later we will see that this Bayesian formulation of our task allows the natural incorporation of an important component to our objective function: the prior probability $\p(\hyp)$, which allows us to incorporate beliefs regarding models such that we avoid degenerate solutions.
    
    \subsection{Maximum likelihood}\label{subsec:maximum-likelihood}
        Recall that the outcome space of model $\hyp$ w.r.t.~observation $\Tobs$ is the set of linear extensions of the poset $P$ defined by the \dfpath~generator $g_\hyp(D)$.
We can then write the likelihood as $\p(\Tobs\mid P) = \p(\Tobs\mid\Tobs\in\ext{P})~\p(\Tobs\in\ext{P})$, i.e.,~the probability of selecting the observed sequence from all linear extensions of $P$.
In principle, this formulation allows us to consider posets for which $D$ is not -- or is not \emph{known to be} -- a linear extension of $P$, but by Def.~\ref{def:dfpath-generator} we have $\p(\Tobs\in\ext{P}) = 1$ by construction.
Assuming a uniform distribution over linear extensions $\ext{P}$ gives

\begin{equation}
    \p(\Tobs\mid\Tobs\in\ext{P})~\p(\Tobs\in\ext{P}) = \frac{1}{|\ext{P}|} \label{eq:likelihood2}
\end{equation}

The assumption of uniformity is common~\cite{huber_fast_2006,huber_near-linear_2014,karzanov_conductance_1991-1,mannila_global_2000}, though not universal~\cite{banks_using_2017}.
We will use Eq.~\ref{eq:likelihood2} as the basis for the likelihood component of our complete objective function described in Sec.~\ref{subsec:objective-function}.

On its own, maximum likelihood is not a sufficient criteria since the likelihood of data $D$ can always be increased with an increasingly narrowly fitted model~\cite{mackay_information_2003}.
In the next section we present a model prior which counterbalances the maximum likelihood term.

    \subsection{Model priors}\label{subsec:model-priors}
        It is common when working with parametric models to assume a uniform distribution over \emph{models}, though not necessarily the parameter space.
Typically a component of the likelihood called the \emph{Occam factor} helps maintain balance between maximizing likelihood and avoiding overfitted parametrizations~\cite{mackay_information_2003}.
Our models are non-parametric, and using uniform model priors would result in our objective function being reduced to likelihood ($1/|\ext{P}|$) alone, resulting in degenerate models in some circumstances.
See Tab.~\ref{tab:model-prior-examples} for an example.
Such a degenerate solution is akin to a clustering task in which each datapoint is assigned to its own cluster.

Typically, overly simplistic models will \emph{under}fit the data.
That is, $\p(D\mid\hyp)$ will tend be low for seen as well as unseen data since the model's probability mass is spread thin across such a large outcome space.
In contrast, our task has an idiosyncratic property that extremely trivial models may perform well in terms of likelihood, despite being completely \emph{uninformative}.
For example, a feature with one outcome across all events would induce the total order $\Tobs$, both on seen and likely on unseen data as well.
This is akin to a prediction task in which a copy of the target variable has unknowingly been included in the data, leading to a trivial classification model.

\begin{figure*}
    \centering
\begin{tikzpicture}[->]

    \tikzstyle{event}=[circle,draw,fill=gray!70,minimum size=5pt,inner sep=0.5pt]

    \def\x{1}
    \def\y{0}

    \draw[fill=gray!10] (-0.15+\x,-0.15+\y) rectangle (0.15+\x,0.6+\y);
    \draw[fill=gray!10] (0.25+\x,-0.15+\y) rectangle (0.95+\x,0.95+\y);

    \node[event] (e7) at (0+\x,0.4+\y) {};
    \node[event] (e34) at (0+\x,0+\y) {};

    \node[event] (e9) at (0.4+\x,0.4+\y) {};
    \node[event] (e30) at (0.4+\x,0+\y) {};

    \node[event] (e1) at (0.8+\x,0.8+\y) {};
    \node[event] (e18) at (0.8+\x,0.4+\y) {};
    \node[event] (e28) at (0.8+\x,0+\y) {};

    \draw (e7) -- (e34);

    \draw (e1) -- (e18);
    \draw (e18) -- (e28);

    \draw (e9) -- (e30);
    \draw (e18) -- (e30);

    \def\x{2.75}
    \def\y{0}

    \draw[fill=gray!10] (-0.15+\x,-0.15+\y) rectangle (0.95+\x,0.95+\y);

    \node[event] (e5) at (0+\x,0.8+\y) {};
    \node[event] (e7) at (0+\x,0.4+\y) {};
    \node[event] (e34) at (0+\x,0+\y) {};

    \node[event] (e9) at (0.4+\x,0.4+\y) {};
    \node[event] (e30) at (0.4+\x,0+\y) {};

    \node[event] (e18) at (0.8+\x,0.4+\y) {};
    \node[event] (e28) at (0.8+\x,0+\y) {};

    \draw (e5) -- (e7);
    \draw (e7) -- (e34);

    \draw (e18) -- (e28);

    \draw (e5) -- (e9);
    \draw (e9) -- (e30);
    \draw (e18) -- (e30);

    \def\x{4.5}
    \def\y{0}

    \draw[fill=gray!10] (-0.15+\x,-0.15+\y) rectangle (0.55+\x,1.35+\y);
    \draw[fill=gray!10] (0.65+\x,-0.15+\y) rectangle (0.95+\x,0.15+\y);

    \node[event] (e2) at (0+\x,1.2+\y) {};
    \node[event] (e5) at (0+\x,0.8+\y) {};
    \node[event] (e7) at (0+\x,0.4+\y) {};
    \node[event] (e34) at (0+\x,0+\y) {};

    \node[event] (e9) at (0.4+\x,0.4+\y) {};
    \node[event] (e30) at (0.4+\x,0+\y) {};

    \node[event] (e28) at (0.8+\x,0+\y) {};

    \draw (e2) -- (e5);
    \draw (e5) -- (e7);
    \draw (e7) -- (e34);

    \draw (e5) -- (e9);
    \draw (e9) -- (e30);

    \def\x{1.75}
    \def\y{1.8}

    \draw[fill=gray!10] (-0.15+\x,-0.15+\y) rectangle (0.95+\x,0.95+\y);

    \node[event] (e5) at (0+\x,0.8+\y) {};
    \node[event] (e7) at (0+\x,0.4+\y) {};
    \node[event] (e34) at (0+\x,0+\y) {};

    \node[event] (e9) at (0.4+\x,0.4+\y) {};
    \node[event] (e30) at (0.4+\x,0+\y) {};

    \node[event] (e1) at (0.8+\x,0.8+\y) {};
    \node[event] (e18) at (0.8+\x,0.4+\y) {};
    \node[event] (e28) at (0.8+\x,0+\y) {};

    \draw (e5) -- (e7);
    \draw (e7) -- (e34);

    \draw (e1) -- (e18);
    \draw (e18) -- (e28);

    \draw (e5) -- (e9);
    \draw (e9) -- (e30);
    \draw (e18) -- (e30);

    \def\x{3.75}
    \def\y{1.8}

    \draw[fill=gray!10] (-0.15+\x,-0.15+\y) rectangle (0.95+\x,1.35+\y);

    \node[event] (e2) at (0+\x,1.2+\y) {};
    \node[event] (e5) at (0+\x,0.8+\y) {};
    \node[event] (e7) at (0+\x,0.4+\y) {};
    \node[event] (e34) at (0+\x,0+\y) {};

    \node[event] (e9) at (0.4+\x,0.4+\y) {};
    \node[event] (e30) at (0.4+\x,0+\y) {};

    \node[event] (e18) at (0.8+\x,0.4+\y) {};
    \node[event] (e28) at (0.8+\x,0+\y) {};

    \draw (e2) -- (e5);
    \draw (e5) -- (e7);
    \draw (e7) -- (e34);

    \draw (e18) -- (e28);

    \draw (e5) -- (e9);
    \draw (e9) -- (e30);
    \draw (e18) -- (e30);

    \def\x{2.75}
    \def\y{3.6}

    \draw[fill=gray!10] (-0.15+\x,-0.15+\y) rectangle (0.95+\x,1.35+\y);
    \draw[fill=gray!10] (1.05+\x,-0.15+\y) rectangle (1.35+\x,1.35+\y);

    \node[event] (e2) at (0+\x,1.2+\y) {};
    \node[event] (e5) at (0+\x,0.8+\y) {};
    \node[event] (e7) at (0+\x,0.4+\y) {};
    \node[event] (e34) at (0+\x,0+\y) {};

    \node[event] (e9) at (0.4+\x,0.4+\y) {};
    \node[event] (e30) at (0.4+\x,0+\y) {};

    \node[event] (e1) at (0.8+\x,0.8+\y) {};
    \node[event] (e18) at (0.8+\x,0.4+\y) {};
    \node[event] (e28) at (0.8+\x,0+\y) {};

    \node[event] (e3) at (1.2+\x,1.2+\y) {};
    \node[event] (e4) at (1.2+\x,0.9+\y) {};
    \node[event] (e6) at (1.2+\x,0.6+\y) {};
    \node[event] (e19) at (1.2+\x,0.3+\y) {};
    \node[event] (e29) at (1.2+\x,0.0+\y) {};

    \draw (e2) -- (e5);
    \draw (e5) -- (e7);
    \draw (e7) -- (e34);

    \draw (e1) -- (e18);
    \draw (e18) -- (e28);

    \draw (e5) -- (e9);
    \draw (e9) -- (e30);
    \draw (e18) -- (e30);

    \draw[dashed,->] (3.4,3.45) -- (2.75,2.8);
    \draw[dashed,->] (3.4,3.45) -- (4,3.25);
    \draw[dashed,->] (2.35,1.6) -- (1.65,1);
    \draw[dashed,->] (2.35,1.6) -- (3.3,1);
    \draw[dashed,->] (4.35,1.6) -- (3.5,1);
    \draw[dashed,->] (4.35,1.6) -- (4.75,1.4);

    \node at (2.25,4.2) {$P$};
    \node at (3.15,4.75) {$Q$};
    \node at (4.35,4.75) {$F$};
    \node at (2.15,2.95) {$R$};
    \node at (4.15,2.95) {$S$};
    \node at (1.0,0.75) {$T$};
    \node at (1.45,0.75) {$U$};
    \node at (3.15,0.75) {$V$};
    \node at (4.85,1.15) {$W$};
    \node at (5.3,0.35) {$X$};

    \node[dotted,draw,align=left,text width=0.78\textwidth] (lower1) at (-1.5,4.2) {$\begin{aligned}|\ext{P}| &= \tbinom{13}{7}|\ext{F}||\ext{Q}| && \!\!\!\!\!\!= \tbinom{13}{7} 5! |\ext{Q}| \\ &\leq \tbinom{14}{7} 5!9! &&~~~~~~~~~~~~~~~~~~~~\approx 1.5 \times 10^{11} \\ &\geq \tbinom{14}{7} 5! &&~~~~~~~~~~~~~~~~~~~~\approx 4.1 \times 10^{5}\end{aligned}$};
    \node[dotted,draw,align=left,text width=0.78\textwidth] (lower2) at (-1.5,2.4) {$\begin{aligned}|\ext{P}| &= \tbinom{14}{7} 5!(|\ext{R}|+|\ext{S}|) \\ &\leq \tbinom{14}{7} 5!(8!+8!) &&~~~~~~~~~~~~\approx 3.3 \times 10^{10} \\ &\geq \tbinom{14}{7} 5!(1+1) ~&&~~~~~~~~~~~~\approx 8.2 \times 10^{5} \end{aligned}$};
    \node[dotted,draw,align=left,text width=0.78\textwidth] (lower3) at (-1.5,0.6) {$\begin{aligned}|\ext{P}| &= \tbinom{14}{7}5!\big(\tbinom{7}{2}|\ext{T}||\ext{U}| + 2|\ext{V}| + \tbinom{7}{1}|\ext{W}||\ext{X}|\big) \\ &\leq \tbinom{14}{7}5!\big(\tbinom{7}{2}2!5! + 2\cdot7! + \tbinom{7}{1}6!\big) &&\!\!\!\!\!\!\!\!\!\!\!\!\!\!\!\!\!\!\!\!\!\!\!\!\!\!\!\!\!\!\!\!\!\!\!\!\!\approx 8.3 \times 10^{9} 
 \\ &\geq \tbinom{14}{7}5!\big(\tbinom{7}{2} + 2 + \tbinom{7}{1}\big) ~~~~~~~~~~~~&&\!\!\!\!\!\!\!\!\!\!\!\!\!\!\!\!\!\!\!\!\!\!\!\!\!\!\!\!\!\!\!\!\!\!\!\!\!\approx 2.1 \times 10^{7} \end{aligned}$};

\end{tikzpicture}
    \caption{Repeated decomposition with bounds from Cor.~\ref{cor:bounds-disjoint} and~\ref{cor:bounds-minelement}. First, for free events $F$ we establish $|\ext{F}| = 5!$, then by disjoint decomposition on $F$ and $Q$ we establish our first bounds. Next, minimal element decomposition is applied to establish tighter bounds using $R$ and $S$.
    Finally, first minimal element then disjoint decomposition is applied resulting in posets $T,U,V,W,X$.}
    \label{fig:decomposition}
\end{figure*}

This pitfall can be avoided by explicitly expressing a preference against both degenerate and unnecessarily complex models.
We propose one such distribution based on the product of the normalized entropies for each feature in the model.
The per-feature score penalizes features with distributions that are deterministic, while taking the product of the $[0,1)$ bounded scores across features discourages the addition of superfluous features.

\vspace{50pt}

\begin{definition}[Normalized Shannon entropy] Let $\Omega_X$ denote the outcome space of discrete random variable $X$, and $H(X)$
the standard Shannon entropy of $X$.
The \emph{normalized Shannon entropy} of $X$ is
    \begin{equation*}
        \Hnorm(X) \coloneqq \frac{\entropy(X)}{1+\max \entropy(X)} = -\frac{\sum_{x\in\Omega_X}\p(x)\log \p(x)}{1+\log |\Omega_X|}
    \end{equation*}
\end{definition}
That is, $\Hnorm(X)\in[0,1)$ is the ratio of the entropy of $X$ to its upper bound with a padding factor to guarantee $\Hnorm(X) < 1$.
It ranges from a deterministic/degenerate (zero surprise) to a uniform (maximal surprise) distribution.

\begin{definition}[$\Hnorm$-based model prior]
    Given a \emph{model} $\hyp\subseteq\dom$, the normalized Shannon entropy prior for $\hyp$ is
    \begin{equation}
        \p(\hyp) \coloneqq \frac{\prod_{\X\in\hyp}\Hnorm(\X)}{\sum_{\hyp\subseteq\dom}\prod_{\X\in\hyp}\Hnorm(\X)}\label{eq:model-prior}
    \end{equation}
    where the denominator is a normalizing constant to ensure $\sum_\hyp\p(\hyp) = 1$.
    \label{def:model-prior}
\end{definition}

Due to the product in Eq.~\ref{eq:model-prior}, any model $\hyp$ containing a trivial feature $\X$ s.t.~$\Hnorm(\X) \approx 0$ will have a prior $\p(\hyp)\approx0$.
Similarly, since $\Hnorm(\X) < 1$ the product will penalize each additional feature.
Additional features will need to improve the likelihood function enough to outweigh the decrease in model prior.
Tab.~\ref{tab:model-prior-examples} shows examples of the effect of outcome distribution on the $\eta$-based model prior.

    \subsection{Posterior odds objective function}\label{subsec:objective-function}
        Combining these components gives a posterior which balances likelihood against our preference for parsimonious and informative models.
For readability we write $\ext{\hyp}^D$ in place of $\ext{g_\hyp(D)}$, the linear extensions of the poset defined by $\hyp$ on $D$.
\begin{align}
    \p(\hyp\mid \Tobs) &=~\frac{\overbrace{\p(\hyp)}^{\text{Eq.~\ref{eq:model-prior}}}~\overbrace{\p(\Tobs\mid\hyp)}^{\text{Eq.~\ref{eq:likelihood2}}}}{\p(\Tobs)} \\
    &= ~\frac{1}{\p(\Tobs)}~\frac{\prod_{\X\in\hyp}\Hnorm(\X)}{\sum_{\hyp\subseteq\dom}\prod_{\X\in\hyp}\Hnorm(\X)}~\frac{1}{|\ext{\hyp}^D|}\label{eq:objective-function}
\end{align}

For multiple independent observations,~$\mathcal{D} \coloneq \{D_1,\ldots,D_N\}$ of the same process, and denoting by $\ext{\hyp}^{D_i}$ the linear extensions of the poset induced by $\hyp$ on observation $D_i$, we can rewrite~Eq.\ref{eq:objective-function} by conditional independence as:

\begin{align*}
    \p(\hyp\mid \mathcal{D}) &=~\frac{\p(\hyp)\prod_{i=1}^N\p(D_i\mid\hyp)}{\p(\mathcal{D})}\\ 
    &= \frac{1}{\p(\mathcal{D})}\frac{\prod_{\X\in\hyp}\Hnorm(\X)}{\sum_{\hyp\subseteq\dom}\prod_{\X\in\hyp}\Hnorm(\X)}~\prod_{i=1}^N\frac{1}{|\ext{\hyp}^{D_i}|}
\end{align*}

Since we are only interested in the \emph{relative} fitness of models, not the exact posterior, it suffices to calculate \emph{posterior odds}, in which case both $\p(\mathcal{D})$ and the normalizing constant stemming from the denominator of~Eq.~\ref{eq:objective-function} cancel:

\begin{align*}
    \frac{\p(\hyp_1\mid \mathcal{D})}{\p(\hyp_2\mid \mathcal{D})} &=\frac{\p(\hyp_1)~\p(\mathcal{D}\mid\hyp_1)}{\p(\hyp_2)~\p(\mathcal{D}\mid\hyp_2)}\\
&=\frac{\prod_{\X\in\hyp_1}\Hnorm(\X)\prod_{i=1}^N|\ext{\hyp_2}^{D_i}|}{\prod_{\X\in\hyp_2}\Hnorm(\X)\prod_{i=1}^N|\ext{\hyp_1}^{D_i}|}
\end{align*}

Furthermore, it suffices to determine if $\p(\hyp_1\mid\mathcal{D}) < \p(\hyp_2\mid\mathcal{D})$ when comparing models, allowing us to take advantage of bounds on the individual terms.

\section{Branch and Bound for Model Selection}\label{sec:strategies-bounds-and-model-selection}
\begin{algorithm*}[ht]
\caption{Bound linear extension count}\label{alg:ext-bounds}
\vspace{5pt}
 \hspace*{\algorithmicindent} \textbf{Input}: $G$ (directed acyclic graph representing a poset), $stopping\_condition$ (e.g.~time elapsed or max iterations)\\
 \hspace*{\algorithmicindent} \textbf{Output}: $l, u$ (upper and lower bounds on a linear extension count of $G$)
\vspace{5pt}
\begin{spacing}{1.08}
\begin{algorithmic}[1]
\Procedure{BoundExtensions}{$G$, $stopping\_condition$}
\If{$stopping\_condition$}\Comment{$\triangleright$ If stopping condition is reached, return naïve bounds}
\State \Return{$1, |G|!$} \label{alg:ext-bounds:stopping-condition}
\Else
    \State $C \gets \textsc{ConnectedComponents}(G)$
    \State $F := \{~\{v\} \mid (\{v\},\emptyset) \in C~\}$ \Comment{$\triangleright$ The set of free vertices~~~~~~~~~~~~~~~~~~}
    \If{$|C| = |F|$} \Comment{$\triangleright$ If graph is edgeless or empty, return exact count}
    \State \Return $|F|!$, $|F|!$
    \ElsIf{$|C| = 1$ and $|F| = 0$} \Comment{$\triangleright$ Connected graph: apply minimal element decomposition}
        \State $l_s, u_s \gets 0, 0$
        \For{$v \in \min G$}
            \State $l, u \gets \textsc{BoundExtensions}(G \setminus \{v\})$
            \State $l_s, u_s \gets l_s + l, u_s + u$ \Comment{$\triangleright$ Sum from Thm.~\ref{thm:minimal}}
        \EndFor
        \State \Return $l_s, u_s$
    \Else \Comment{$\triangleright$ Two or more connected graphs: apply disjoint decomposition}
        \If{$F \neq \emptyset$}
            \State $l_1, u_1 \gets |F|!, |F|!$
            \State $l_2, u_2 \gets \textsc{BoundExtensions}(C \setminus F)$
            \State $k \gets |F|$
        \Else
            \State $H \gets $ some component in $C$
            \State $l_1, u_1 \gets \textsc{BoundExtensions}(H)$
            \State $l_2, u_2 \gets \textsc{BoundExtensions}(G \setminus H)$
            \State $k \gets |H|$
        \EndIf
        \State \Return $\binom{|G|}{k} \cdot l_1 \cdot l_2, \binom{|G|}{k} \cdot u_1, u_2$ \Comment{$\triangleright$ Formula from Thm.~\ref{thm:disjoint}}
    \EndIf
\EndIf
\EndProcedure
\end{algorithmic}
\end{spacing}
\end{algorithm*}

Exact computation of the linear extension count (the denominator  in Eq.~\ref{eq:likelihood2}) can be very expensive, depending on the structure of the partial order.
However, we can use several properties of partial orders and linear extension count to construct a greedy, divide-and-conquer strategy for establishing bounds and incorporate this into a branch-and-bound algorithm for model selection.

\begin{algorithm*}[ht]
\caption{Expand or prune model node (\emph{for one sample $D$})}\label{alg:bnb}
\vspace{5pt}
\hspace*{\algorithmicindent} \textbf{Input}: $\hyp$ (a model), $D$ (data sample), $best\_score$ (current best score)\\
 \hspace*{\algorithmicindent} \textbf{Output}: Whether to prune or expand subspace resulting from $\hyp$.
 \vspace{5pt}
 
\begin{spacing}{1.08}
\begin{algorithmic}[1]
\Procedure{ExpandOrPrune}{$\hyp, D, best\_score$}
\State $G^* \gets \textsc{BuildPoset}(D,\hyp^*)$\Comment{$\triangleright$ Most restrictive reachable poset, see Eq.~\ref{eq:model-closure}}
\State $l^*, u^* \gets \textsc{BoundExtensions}(G^*)$
\If {$\frac{\p(\hyp)}{l^*} < best\_score$} \Comment{$\triangleright$ Upper bound on achievable score in subspace, see Cor.~\ref{cor:best-possible}}\label{alg:bnb:prune} 
\State \Return $prune$
\Else
    \State $G \gets \textsc{BuildPoset}(D,\hyp)$
    \State $l, u \gets \textsc{BoundExtensions}(G, stopping\_condition)$
    \If{$\frac{\p(\hyp)}{u} > best\_score$} \Comment{$\triangleright$ $\hyp$ is provably better: lower bound higher than current best\label{alg:bnb:better}}
        \State $best\_score \gets \frac{\p(\hyp)}{u}$
        \If{$l = u$} 
        \State Record $\hyp$ as current best model
        \Else 
        \State Add $\hyp$ to models to re-estimate later
        \EndIf
    \ElsIf{$\frac{\p(\hyp)}{\log(l)} > best\_score > \frac{\p(\hyp)}{\log(u)}$} \Comment{$\triangleright$ $\hyp$ might be better: bounds straddle current best} \label{alg:bnb:maybe}
        \State Add $\hyp$ to models to re-estimate later
    \EndIf \Comment{$\triangleright$ $\hyp$ is provably worse: upper bound lower than current best} \label{alg:bnb:worse}
    \State \Return $expand$
\EndIf
\EndProcedure
\end{algorithmic}
\end{spacing}
\end{algorithm*}

\subsection{Bounding linear extension count}\label{subsec:bounding-linear-extension-count}
Let $P = (E,\prec)$ be a poset. If $P$ is a total order, there is only a single linear extension, namely the order itself, and if $\prec~= \emptyset$, all permutations of $E$ are linear extensions. Hence, it is easy to see that
 $1 \leq |\ext{P}| \leq |E|!$
We can combine this naïve bound  with the following two rules for recursive decomposition of posets.
\begin{theorem}[Disjoint decomposition]
    Given poset $P$ partitioned by $\{P_1, P_2\}$ s.t.\ no $ a\in P_1$ is comparable to any $b \in P_2$, then
    \[|\ext{P}| = \binom{|P|}{|P_1|}|\ext{P_1}||\ext{P_2}|\]
    \label{thm:disjoint}
\end{theorem}
\begin{proof}
    See~\cite{mohring_computationally_1989}.
\end{proof}
\begin{corollary}
    Given poset $P$ partitioned by $\{P_1, P_2\}$ s.t.\ no $ a\in P_1$ is comparable to any $b \in P_2$, then by applying naïve bounds to Thm.~\ref{thm:disjoint} we can establish that
    
    \begin{equation*}
        \binom{|P|}{|P_1|} \leq \binom{|P|}{|P_1|}|\ext{P_1}||\ext{P_2}| \leq \binom{|P|}{|P_1|}|P_1|!|P_2|!
    \end{equation*}
    \label{cor:bounds-disjoint}
\end{corollary}
\begin{theorem}[Minimal element decomposition]
    Given a poset $P$ s.t.~${|P|>1}$ with set of minimal elements $\min P$ s.t.~$a \in \min P$ and $b \prec a$ implies $b = a$, then
    \[|\ext{P}| = \sum_{x \in \min P} |\ext{P \setminus \{x\}}|\]
    \label{thm:minimal}
\end{theorem}
\begin{proof}
    See~\cite{talvitie_counting_2018}.
\end{proof}
\begin{corollary}
    Given poset $P$ s.t.~$|P|>1$ with set of minimal elements $\min P$, then by applying naïve bounds to Thm.~\ref{thm:minimal} we can establish that
    \begin{equation*}
         |\min P| = \sum_{x \in \min P} 1 \leq \sum_{x \in \min P} |\ext{P \setminus \{x\}}| \leq \sum_{x \in \min P} |P \setminus \{x\}|!
    \end{equation*}
    \label{cor:bounds-minelement}
\end{corollary}
These decomposition rules can be recursively applied, substituting a new lower bound estimate at each step as the resulting posets shrink in size, either by removal of one minimal element, or by splitting into two or more disjoint posets.
The procedure is described in Alg.~\ref{alg:ext-bounds}, and in Fig.~$\ref{fig:decomposition}$ we illustrate decomposition of the poset from Fig.~\ref{fig:poset-running}, with naïve bound calculations from Cor.~\ref{cor:bounds-disjoint} and ~\ref{cor:bounds-minelement}.

\subsection{Bounds for pruning}

The search space over models is exponential in the number of features and calculating $\p(\hyp | \Tobs)$ for each model is expensive due to the $|\ext{P_\hyp}|$ term in Eq.~\ref{eq:likelihood2}.
Two factors, respectively, help tackle these challenges: pruning large portions of the model space; and quick evaluation for each model using bounds on, rather than exact, linear extension counts.
The latter will allow us to either reject a model, or mark it for revisit and reestimation, meanwhile potentially updating the current best score.

When maximizing an objective function $f(\hyp)$, pruning in branch and bound requires a bounding function $b$ that provides a closer and closer upper bound on $f$ as more constraints are added to set of candidate solutions (models), and agrees with $f$ for solutions at the end of the search space.
Formally, $b(\hyp) \geq f(\hyp)$ for all $\hyp$, $b(\hyp) = f(\hyp)$ for all leaves in the search space, and $b(\hyp_i) \leq b(\hyp_j)$ if $\hyp_j$ is the parent state (model) of $\hyp_i$.

Our full objective function $\p(\hyp | D)$ is intentionally non-convex, with the terms $\p(\hyp)$ and $\p(\Tobs | \hyp)$ pulling in opposing directions.
Nevertheless, we can formulate a convex bounding function that takes advantage of upper bounds on each term, assuming a monotone branching rule w.r.t.~model inclusion.
That is, that child states in our search space are generated by adding features to a candidate model.

\begin{theorem}[$\p(\hyp)$ is strictly antitonic] 

\[\p(\hyp_i) > \p(\hyp_j) \text{~for~} \hyp_i \subset \hyp_j.\]
\end{theorem}

\begin{proof}
    We have defined $\p(\hyp_i)$ to be the product of $\Hnorm(\X)$ for each $\X \in \hyp_i$. Since $\Hnorm(\X)$ is bounded in $[0,1)$, adding an extra term to this product cannot increase its value. We have stated that $\hyp_i \subset \hyp_j$, so any feature $\X \in \hyp_j \setminus \hyp_i$ can only reduce $\p(\hyp_j)$ relative to $\p(\hyp_i)$.
\end{proof}

\begin{lemma}[$\hyp\subseteq\hyp'$ implies $g_\hyp(D) \subseteq g_{\hyp'}(D)$]
Let $D$ be a total order and let $\hyp$ and $\hyp'$ be models s.t.~$\hyp\subseteq\hyp'$.
Then the set of order relations $g_\hyp(D)$ induced by $\hyp$ via the \dfpath~ generator as defined in Def.~\ref{def:dfpath-generator} will be a subset of the order relations $g_{\hyp'}(D)$ induced by $\hyp'$.

\label{eq:lemma-1}
\end{lemma}

\begin{proof}
    Follows immediately from Def.~\ref{def:feature-relation},\ref{def:derived-feature-relation},\ref{def:dfpath-generator}.
\end{proof}

\begin{lemma}[$\prec \subseteq \prec'$ implies $\ext{\prec} \supseteq \ext{\prec'}$]
Let $\prec$ and $\prec'$ denote order relations over the same set.
For the respective sets of linear extensions, $\ext{\prec}$ and $\ext{\prec'}$, the fact that $\prec \subseteq \prec'$ implies that $\ext{\prec} \supseteq \ext{\prec'}$.
We say the latter order is more \emph{restrictive} than the former.

\label{eq:lemma-2}
\end{lemma}

\begin{figure*}[ht]
\includegraphics[width=\textwidth]{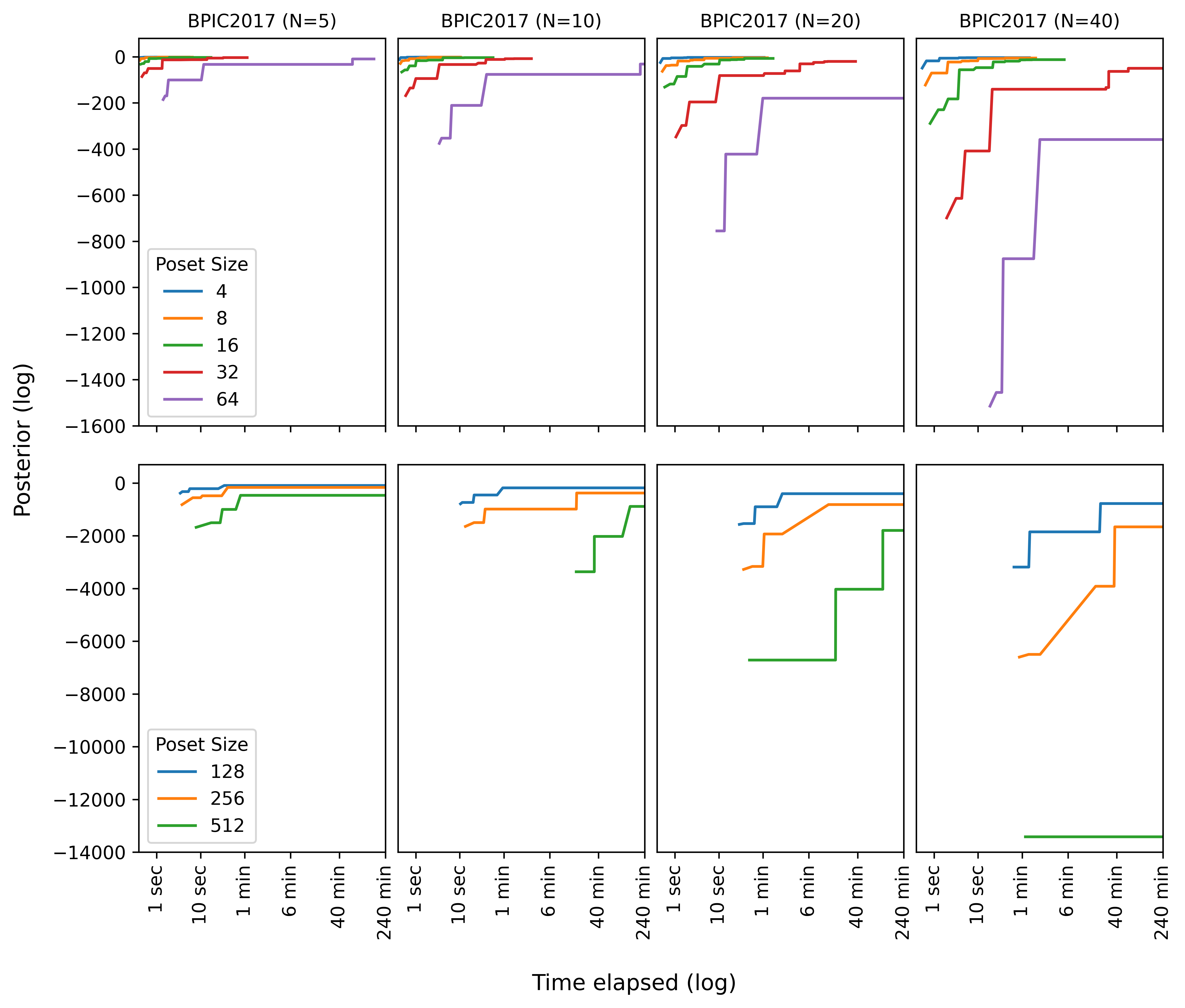}
    \caption[]{Convergence results after four hours across varying number of samples ($5 \leq N \leq 40$) and poset/sample size ($4 \leq |D| \leq 512$).
    Lines begin at first completed estimate; those that end indicate the algorithm finished before four hour cutoff.}
\end{figure*}

\begin{proof}
    Assume that $\prec = \prec'$, clearly this means that $\ext{\prec} = \ext{\prec'}$.
    It remains to show that when $\prec \subset \prec'$, that no total order $D$ exists s.t. $D \in \ext{\prec'}$ and $D \notin \ext{\prec}$.
    Assume such a $D$ exists and that $\prec \subset \prec'$.
    Since $D \notin \ext{\prec}$ this means that there exists some relation $(a,b) \in \prec$ that prevents $D$ from extending $\prec$ while $(a,b) \notin \prec'$ since $D$ extends $\prec'$.
    This means that $\prec \not \subseteq \prec'$ and we have a contradiction of the premise.
\end{proof}

The following bound on $\p(D|\hyp)$ assumes that from the state associated with $\hyp$, the set of features in models reachable from this state is known, allowing us to determine the most restrictive reachable poset.
This assumption is consistent with standard algorithms for recursively generating powersets.

\begin{theorem}[Upper bound on $\p(D | \hyp)$ is antitonic]
    Let $S_\hyp$ denote the portion of the search space reachable from state $\hyp$ (inclusive $\hyp$), with $\hyp\subseteq\hyp'$ for any $\hyp'\in S_\hyp$ and define 

    \begin{equation}
        \hyp^*:=\bigcup_{\hyp'\in S_\hyp} \hyp'
        \label{eq:model-closure}
    \end{equation}

    which represents the largest, most restrictive model reachable from $S_\hyp$. Then 
    \[{\p(D | \hyp) \leq \frac{1}{|\ext{\hyp^*}|}} \text{~and~} \frac{1}{|\ext{\hyp_i^*}|} \geq \frac{1}{|\ext{\hyp_j^*}|} \text{~for~} \hyp_i \subset \hyp_j.\]
\label{eq:bounds-on-likelihood}
\end{theorem}

\begin{proof}
    By Lem.~\ref{eq:lemma-1} and~\ref{eq:lemma-2}, since $\hyp \subseteq \hyp^*$ we have that 
    \[g_\hyp(D) \subseteq g_{\hyp^*}(D)\]
    which implies \[|\ext{\hyp}^D| \geq |\ext{\hyp^*}^D|\] 
    and equivalently 
    \[\frac{1}{|\ext{\hyp}^D|} \leq \frac{1}{|\ext{\hyp^*}^D|}.\]
    By the same reasoning, we have that $\hyp_i \subseteq \hyp_j$   implies 
    \[\frac{1}{|\ext{\hyp_i^*}|} \geq \frac{1}{|\ext{\hyp_j^*}|}.\]
\end{proof}

\begin{table*}[ht]
    \caption{Models discovered for configurations that completed within four hours. Feature {\small\texttt{case:concept:name}} refers to the loan application ID, and {\small\texttt{org:resource}} to the employee executing an event.}
    \centering
    \begin{tabular}{c|c|c|c|c|c|c|c|c|c|c|c|c|c|c|c|c}
        &Samples&\multicolumn{4}{c|}{N=5} & \multicolumn{4}{c|}{N=10} & \multicolumn{4}{c|}{N=20} & \multicolumn{3}{c}{N=40} \\
        \hline
        \hline
        &Poset size & 4 & 8 & 16 & 32 & 4 & 8 & 16 & 32 & 4 & 8 & 16 & 32 & 4 & 8 & 16 \\
        \hline
        \hline
        \multirow{5}*{\rotatebox{90}{Feature~~}}\hspace{5mm}&\texttt{Action} & & & & & & & & & & & & & & & \\
        \cline{2-17}
        &\texttt{ApplicationType} & \checkmark & \checkmark & \checkmark & \checkmark & \checkmark & \checkmark & \checkmark & \checkmark & \checkmark & \checkmark & \checkmark & \checkmark & \checkmark & \checkmark & \checkmark \\
        \cline{2-17}
        &\texttt{case:concept:name} & & & & & & & & & & & & & & & \\
        \cline{2-17}
        &\texttt{EventOrigin} & \checkmark & \checkmark & \checkmark & \checkmark & \checkmark & \checkmark & \checkmark & \checkmark & \checkmark & \checkmark & \checkmark & \checkmark & \checkmark & \checkmark & \checkmark \\
        \cline{2-17}
        &\texttt{org:resource} & & & & & & & & & & & & \checkmark & & & \\
        \cline{2-17}
        &\texttt{LoanGoal} & & & & & & & \checkmark & & \checkmark & \checkmark & \checkmark & \checkmark & \checkmark & \checkmark & \checkmark \\
        \cline{2-17}
        &\texttt{lifecycle:transition} & & & & & & & & & \checkmark & & & \checkmark & & & \checkmark \\
    \end{tabular}
    \vspace{2mm}
    \label{tab:models}
\end{table*}

\begin{figure*}[ht]
\centering
\begin{subfigure}[b]{\textwidth}
   \includegraphics[width=1.0\textwidth,trim={22mm 233.2mm 106mm 33mm}, clip]{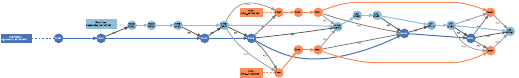}
   \caption{}
   \label{fig:their-ekg-681547497}
\end{subfigure}
\begin{subfigure}[b]{\textwidth}
   \includegraphics[width=1\linewidth]{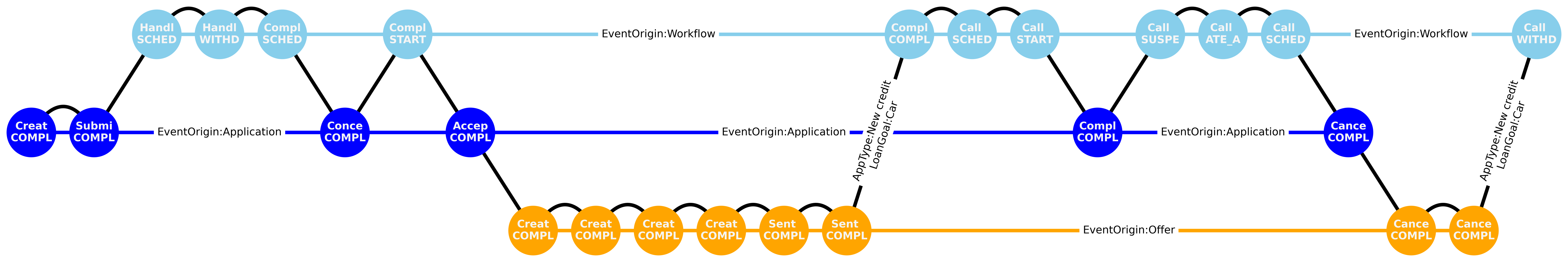}
   \caption{}
   \label{fig:our-ekg-apptype-681547497}
\end{subfigure}
\caption{Event knowledge graphs for the loan application number 681547497 from the BPIC 2017 event log. Fig.~(\subref{fig:their-ekg-681547497}) is the handbuilt EKG from~\cite{esser_multi-dimensional_2021}, Fig.~(\subref{fig:our-ekg-apptype-681547497}) is derived from the model most commonly discovered by our algorithm.}
    \label{fig:full-ekg-comparison-1}
\end{figure*}

\begin{corollary}
    Let $S_\hyp$ be defined as in Thm.~$\ref{eq:bounds-on-likelihood}$, then 
    \[\p(\hyp)\p(D | \hyp) < \frac{\p(\hyp)}{|\ext{\hyp^*}^D|}\]
    and
    \[\frac{\p(\hyp_i)}{|\ext{\hyp_i^*}^D|} > \frac{\p(\hyp_j)}{|\ext{\hyp_j^*}^D|}\]
    for $\hyp_i \subset \hyp_j$.
    By extension to $N$ independent observations
    \[\p(\hyp_i)\prod_{k=1}^N\frac{1}{|\ext{\hyp_i^*}^{D_k}|} > \p(\hyp_j)\prod_{k=1}^N\frac{1}{|\ext{\hyp_j^*}^{D_k}|} \text{~for~} \hyp_i \subset \hyp_j.\]
    \label{cor:best-possible}
\end{corollary}

This bound is integrated into Alg.~\ref{alg:bnb} on line~\ref{alg:bnb:prune} using a lower bound on $|\ext{\hyp^*}^{D}|$.

\subsection{Bounds for model evaluation}

When pruning is not possible, we can still harness bounds on $\p(\hyp)\p(\mathcal{D} | \hyp)$ to minimize the time spent evaluating a candidate model.
Alg.~\ref{alg:bnb} takes advantage of the following three cases.

The first (line~\ref{alg:bnb:better}) allows us to update the current best score and model when its lower bound is above the current best score.
The second (line~\ref{alg:bnb:maybe}) allows us to mark a model as a potential candidate to be revisited when the current best score lies between its upper and lower bound.
The third (line~\ref{alg:bnb:worse}) allows us to dismiss a model as soon as its upper bound is below the current best score.

Note that in cases where bounds have not converged, the model will be marked to be revisited for a more precise bound calculation.
Eagerly updating the lower bound estimate on the best possible score in this manner helps improve the convergence rate.
Quickly marking models (by extension portions of search space) for revisit and moving on means that when models are revisited, they are more likely to be dismissed when compared to the now greatly improved best score, or to require less precision in bound estimation before eventual dismissal.

\section{Experiments}\label{sec:evaluation}
\begin{figure*}[ht]
\centering
\begin{subfigure}[b]{\textwidth}
      \includegraphics[width=1.0\textwidth,trim={22mm 248mm 113mm 20mm}, clip]{graphics/bpic2017_ekgs_2.pdf}
   \caption{}
   \label{fig:their-ekg-889180637}
\end{subfigure}
\begin{subfigure}[b]{\textwidth}
   \centering\includegraphics[width=1\linewidth]{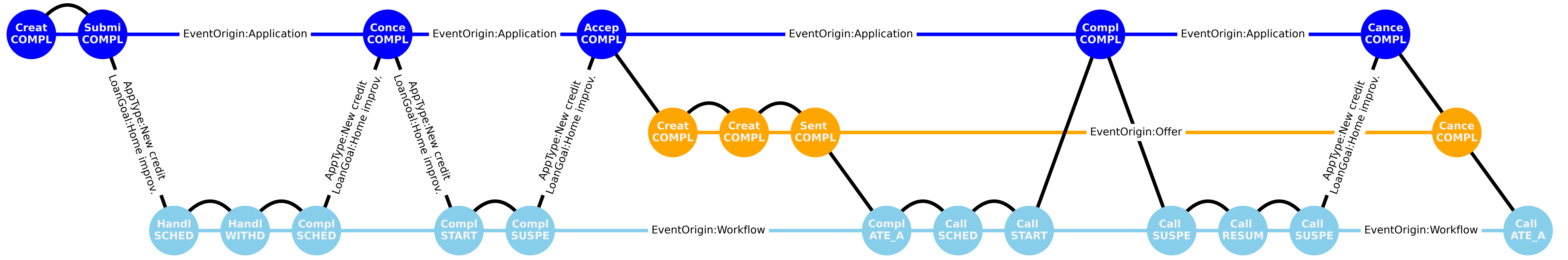}
   \caption{}
   \label{fig:our-ekg-apptype-889180637}
\end{subfigure}
\caption{Event knowledge graphs for the loan application number 889180637 from the BPIC 2017 event log. Fig.~(\subref{fig:their-ekg-889180637}) is the handbuilt EKG from~\cite{esser_multi-dimensional_2021}, Fig.~(\subref{fig:our-ekg-apptype-889180637}) is derived from the model most commonly discovered by our algorithm.}
    \label{fig:full-ekg-comparison-2}
\end{figure*}

We evaluate our approach both in terms of quality of output and runtime performance on the well-known BPIC 2017 event log~\cite{BPIC2017} which is commonly used as a running example in the EKG literature and for which handbuilt EKGs exist.
It contains a total of 19 features, resulting in a search space of 524,288 possible models built from atomic features.

The algorithm is implemented using a prefix-based, breadth-first branching rule for exploring the model space, which satisfies the monotonic condition on branching rules stipulated in Thm.~\ref{eq:bounds-on-likelihood}.
The subroutine for estimating bounds on linear extensions of a poset uses a depth-first approach with iterative deepening.
The stopping condition (see line~\ref{alg:ext-bounds:stopping-condition} in Alg.~\ref{alg:ext-bounds}) is based on time elapsed\footnote{Due to the time required to roll back the recursive call stack, actual time elapsed will often exceed the user-defined stopping condition.}, set to 1 second in the first rounds, but increasing exponentially when revisiting models that were marked for reestimation on earlier passes.
This preliminary implementation considers only atomic features since the implications for monotonicity properties of including derived features was not immediately clear.
This is a top-priority for future work.

To assess runtime performance, we monitored the increase in the current best score over time for varying numbers ($N$) and sizes ($|D|$) of samples, i.e., posets.
The size of a sample will typically have a drastic impact on the complexity of estimating bounds on linear extensions, excepting easily computable special cases such as chains or free posets.
Since samples are independent, the \emph{number} of samples will have a linear impact on the time required to evaluate a model, or sublinear with parallelization.
There will be some variance in the time required to evaluate each sample when the poset structure differs, but this should amortize across samples.

Increased sample size has the important downstream effect of delaying the establishment of bounds early on which - via pruning and model dismissal - allows rapid convergence.
Indeed, we see that the poorest performing parametrizations (sample size 128 to 512), take a very long time to establish the very first estimate after which they begin the same rapid convergence pattern.

To assess the quality of the discovered models, we compare discovered models with handbuilt models from~\cite{esser_multi-dimensional_2021}. Tab.~\ref{tab:models} shows an overview of the models discovered by those configurations that completed within an hour. Figs.~\ref{fig:full-ekg-comparison-1},\ref{fig:full-ekg-comparison-2}, and~\ref{fig:full-ekg-comparison-3} show comparisons of a selection of handbuilt EKGs visualized in~\cite{esser_multi-dimensional_2021} - namely those for application numbers 681547497,  889180637, and 2014483796 - with the corresponding EKGs for the most commonly discovered model.

\begin{figure*}[ht]
\centering
\begin{subfigure}[b]{\textwidth}
   \includegraphics[width=1.0\textwidth,trim={22mm 200.3mm 103mm 68mm}, clip]{graphics/bpic2017_ekgs_2.pdf}
   \caption{}
   \label{fig:their-ekg-2014483796}
\end{subfigure}
\begin{subfigure}[b]{\textwidth}
   \includegraphics[width=1\linewidth]{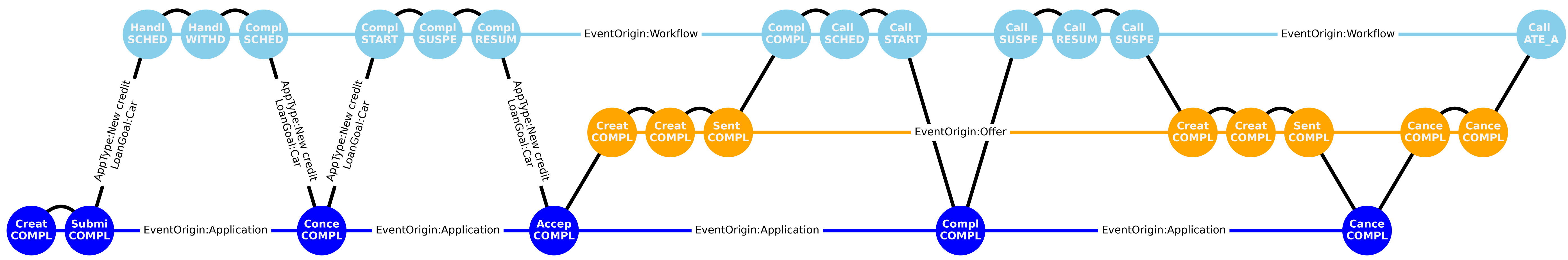}
   \caption{}
   \label{fig:our-ekg-apptype-2014483796}
\end{subfigure}
\caption{Event knowledge graphs for the loan application number 2014483796 from the BPIC 2017 event log. Fig.~(\subref{fig:their-ekg-2014483796}) is the handbuilt EKG from~\cite{esser_multi-dimensional_2021}, Fig.~(\subref{fig:our-ekg-apptype-2014483796}) is derived from the model most commonly discovered by our algorithm.}
    \label{fig:full-ekg-comparison-3}
\end{figure*}

Figs.~\ref{fig:our-ekg-apptype-681547497},\ref{fig:our-ekg-apptype-889180637}, and~\ref{fig:our-ekg-apptype-2014483796} show the EKGs for model {\small\{\{\texttt{EventOrigin}\},\{\texttt{ApplicationType}\},\{\texttt{LoanGoal}\}\}} which shows a close, though not exact, correspondence with the handbuilt EKG.
In particular, that {\small\texttt{EventOrigin}} was identified is notable, since this is a core entity in the handbuilt model and defines the three colored ``swimlanes'' corresponding to the values: \textit{Application}, \textit{Workflow} and \textit{Offer}.
The {\small\texttt{ApplicationType}} and {\small\texttt{LoanGoal}} entities are overlapping within cases, and have the effect of binding those events together that are connected by the handbuilt derived entities: \textit{AW}, \textit{WO}, and \textit{AO} in Figs.~\ref{fig:our-ekg-apptype-681547497},\ref{fig:our-ekg-apptype-889180637}, and~\ref{fig:our-ekg-apptype-2014483796}.
One clear discrepancy between discovered models and the handbuilt EKG is that our algorithm fails to select the \texttt{OfferID} feature, resulting in two offers being combined into one \dfpath in applications 681547497 and 2014483796 (see Figs.~\ref{fig:our-ekg-apptype-681547497},\ref{fig:our-ekg-apptype-889180637} and~\ref{fig:our-ekg-apptype-2014483796}).

\section{Conclusion}\label{sec:conclusion}
We have presented the first algorithm for the automated discovery of event knowledge graphs (EKG) from uncurated data, and shown that it can discover models congruent with handbuilt EKGs within a reasonable runtime.
Our approach builds on a principled, probabilistic framing: viewing the outcome space of an EKG as the linear extensions it admits when interpreted as poset.

This framing lays a foundation for several natural extensions, in particular in extending or redefining the notion of outcome space.
This could include pairwise event orders, a concept ubiquitous in literature on linear extensions.
And importantly, adding event features themselves to the outcome space would permit the discovery of patterns at higher levels of abstraction than individual events, including predictive models whose utility would be intuitively demonstrable.

While the theoretical results presented hold for both atomic and derived features, the latter were omitted from our implementation.
Incorporating these is a top priority for future work and requires specifying appropriate branching rules that preserve monotonicity and take advantage of the subsumption of atomic feature relations by derived.

Another crucial question to be addressed concerns the quantitative evaluation of discovered models.
Given a hand-built model, how do we quantitatively measure closeness of a discovered model, e.g.~using graph matching, edit distance, or other similarity metrics?
How can discovered EKGs be evaluated without reference to handbuilt models, e.g.~by comparison to held-out portions of the dataset: whether events or relations?
This would permit a more traditional framing of the task as learning problem --
in particular, might we harness modern graph encoding techniques, e.g.~graph neural networks, to tackle this problem or a reasonable reformulation of it?
While the real-world utility of any EKG discovery techniques will need to be evaluated w.r.t.~concrete use cases, the results presented here should provide a well-founded and extensible point of reference, and a baseline for alternative approaches.


%





\ifCLASSOPTIONcaptionsoff
  \newpage
\fi



%

\bibliographystyle{splncs04}
\bibliography{bibliography.bib}

%








\end{document}